\documentclass[11pt]{article}
\usepackage[utf8]{inputenc}
\usepackage{mystyle}

\begin{document}

\title {\huge Representation Theorem for Matrix Product States}

\author
{Erdong Guo\thanks{University of California, Santa Cruz, email: \texttt{eguo1@ucsc.edu}} \qquad David Draper\thanks{University of California, Santa Cruz, email: \texttt{draper@ucsc.edu}} 
}

\date{}

\maketitle

\begin{abstract}
In this work, we investigate the universal representation capacity of the Matrix Product States (MPS) from the perspective of boolean functions and continuous functions. We show that MPS can accurately realize arbitrary boolean functions by providing a construction method of the corresponding MPS structure for an arbitrarily given boolean gate. Moreover, we prove that the function space of MPS with the scale-invariant sigmoidal activation is dense in the space of continuous functions defined on a compact subspace of the $n$-dimensional real coordinate space $\mathbb{R}^{n}$. We study the relation between MPS and neural networks and show that the MPS with a scale-invariant sigmoidal function is equivalent to a one-hidden-layer neural network equipped with a kernel function. We construct the equivalent neural networks for several specific MPS models and show that non-linear kernels such as the polynomial kernel which introduces the couplings between different components of the input into the model appear naturally in the equivalent neural networks. At last, we discuss the realization of the Gaussian Process (GP) with infinitely wide MPS by studying their equivalent neural networks. 
%of the infinitely wide MPS.   
\end{abstract}

%!TEX root = main.tex

\section{Introduction}\label{sec: intro}

\subsection{Background}

The framework of Tensor Networks is a powerful diagrammatic language used to represent the many-body quantum systems. In this graphical language, tensor nodes with $k$ left edges and $l$ right edges represent the $(k, l)$ order tensors living in the product space $\mathcal{H}^{\otimes k}\otimes\mathcal{H}^{*\otimes l}$, where $\mathcal{H}$ stands for the Hilbert space and $\mathcal{H}^{*}$ stands for the dual space. The connecting edges between two tensor nodes represent the contraction of duplicated indices in two tensor nodes. With these simple rules, a notation system of computation by quantum states is developed \citep{penrose1971applications} and this tool box of tensor computation with graphical symbols is generally called 'Tensor Networks'. 

People from different communities: Quantum Information, Condensed-Matter Physics, Applied Mathematics, Computer Science, etc., find Tensor Networks is really useful and study it from different perspectives. In Quantum Information area, people suggest to use the the evolution of the quantum objects to perform quantum computation. More precisely, the input and output of the computation machine are the quantum states and the quantum gates are the S matrices which describe the evolution of the quantum systems \citep{feynman1986quantum, deutsch1989quantum}. Since all the components of this 'computation instrument' are actually different order tensors which are defined on the tensor products of countable Hilbert spaces and the dual spaces, this computation process can be represent by a tensor networks diagram called 'Quantum circuits'. To study the many-body system, people use the tensor networks to approximate the many-body wave functions. The idea is to use the Schmidt Decomposition (SD) to decompose the n-qubit quantum state $\ket{\psi}$ as the product of the lower quantum states $\ket{\psi_{A}}$ and $\ket{\psi_{B}}$ living in the partition $\{A, B\}$ of the original space \citep{vidal2003efficient, verstraete2004density}. By iterating this step, the original n-qubit quantum state can be written as the product of $n$ tensors and $n-1$ vectors. Actually this is closely related to the idea of Density Matrix Renormalization Group (DMRG) which is widely used in the numerical simulation of the quamtum many-body systems \citep{white1992density}. To extend this idea to higher dimensions,  Projected Entangled-Pair State (PEPS) and the higher dimensional version of one dimensional method: such as Tree Tensor Networks (TTN) and Multi-scale Entanglement Renormalization Ansatz (MERA) are proposed \citep{vidal2007classical, cirac2009renormalization, orus2014advances, evenbly2011tensor, cichocki2016tensor}.
Please see \citep{biamonte2019lectures, biamonte2017tensor, orus2014practical} to get more knowledge on the tensor networks.
%The readers who want to know more about tensor networks can see \citep{biamonte2019lectures, biamonte2017tensor, orus2014practical}.

Recently, tensor networks method also got attention of people in the machine learning community and several interesting statistical learning models based on tensor networks are proposed in regression, classification and density estimation problems \citep{stoudenmire2016supervised, novikov2016exponential, han2018unsupervised, glasser2018supervised, gallego2017language}. A natural question gets people's interest is why the Tensor Networks work so well? To make the question more precise, we can ask what is the representation capacity of the Tensor Networks as an algebraic machinery. A lot of progress has been made in this topic. In \citep{perez2006matrix}, the properties of the Matrix Product States (MPS), one dimensional tensor networks are carefully studied. It is shown that arbitrary quantum states can be represented by MPS with open boundary condition in canonical form. Actually from quantum theory, we know that arbitrary quantum states can be represented by the product of a series of three order tensors and vectors \citep{vidal2007classical}. In \citep{biamonte2011categorical}, the authors proposed a general method to factorize the quantum state into tensor networks blocks by category theory. Moreover, they also study the boolean quantum states, namely all the coefficients of the states is binary values under local basis. 

\subsection{Our Contributions}

In our work, we focus on the representation power of MPS . Although MPS is inspired by the study of quantum many-body physics and has interesting physical properties, 
%we will get rid of the physical meaning of the MPS and just treat the MPS as a black box which can be used to describe the map between the predictors $\mathbf{X}$ and labels $\mathbf{y}$,
we only treat the MPS as a algebraic machinery which receives the data $\mathbf{X}$ and outputs the response $\mathbf{\Psi(\mathbf{x})}$ \footnote{In a specific problem, the response can be a distribution or a discriminant function according to the modeling method used.}. As is known, when $\{\Phi^{s_{1}\cdots s_{n}}\}$, $s_{i}\in\{1, \cdots, d\}$ form a complete basis of the product Hilbert space. However, this is not practical in the application case since the complete basis are usually infinite dimensional which can not be implemented in the real case. 

So we want to ask how large is the function space 'spanned' by the MPS without assuming the kernel function to be complete if mild assumptions such as measurable, continuous activation are assumed? To answer this question, we firstly prove that arbitrary boolean functions can be represented by the MPS. We provide a constructive proof in which a general method to construct the MPS for an arbitrarily given boolean function is proposed. Moreover, we show that the function space 'spanned' by the mildly modified MPS is dense in the space of continuous functions on the unit cube, with respect to the supremum norm. In our proof, we use the similar idea in \citep{cybenko1989approximation, hornik1991approximation, stinchcombe1990approximating} where the universality of the one-hidden-layer neural networks with sigmoidal activation is proved. To be precise, in these two works, the authors show the existence of a specific one-hidden layer neural network profile which is arbitrarily near to the given continuous function on the unit cube in the sense of the supremum norm. In our work, the key steps are showing the discriminatory property of the scale-invariant sigmoidal function and the linearity of our mildly modified MPS. In the next section, we will introduce our mildly modified MPS and also discuss its properties. 

\subsection{Set-up of the MPS Model}
\label{sec: set-up}
%The tensor network is the network with tensors as nodes, and the edges between two nodes are the contracted indices. Tensor networks can have different topological structures and the most well known structures are $1$D Matrix Product States (MPS), Tensor Trains (TT), Tree Tensor Networks (TTN), the Multi-scale Entanglement Renormalization Anzats (MERA), Projected Entangled Pair States (PEPS) and various other renormalization methods~\citep{stoudenmire2016supervised, novikov2016exponential, }.
%\vspace{5pt}
%{\bf \noindent Set-up.} 
In supervised learning task \citep{stoudenmire2016supervised}, the response variable of the MPS model is originally defined as 
\begin{align}
\label{eq: pure_mps}
    \Psi^{l}(\mathbf{x}|\mathbf{w}, \mathcal{B}) = \sum_{\{\alpha, s\}}A^{s_{1}}_{\alpha_{1}\alpha_{2}}\cdots A^{ls_{i}}_{\alpha_{i}\alpha_{i+1}}\cdots A^{s_{n}}_{\alpha_{n}\alpha_{1}}\Phi^{s_{1}\cdots s_{n}}(\mathbf{x}),
\end{align}
where $\Phi(\mathbf{x})$ is called kernel function defined as the tensor product of a series of map $\phi^{s_{i}}(x_{i})$ of each component of the input $x_{i}$, namely
\begin{align}
&\Phi^{s_{1}\cdots s_{n}}(\mathbf{x}) = \phi^{s_{1}}(x_{1})\otimes\cdots\phi^{s_{i}}(x_{i})\cdots\otimes\phi^{s_{n}}(x_{n}),\\
&\phi^{s_{i}}(x_{i}): x_{i}\to \mathbb{R}^{d}. 
\end{align}
\label{eq: kernel_function}
It is obvious that $\Psi^{l}(\mathbf{x})$ is linear with respect to $\Phi^{s_{1}\cdots s_{n}}(\mathbf{x})$, and also linear with respect to $\phi^{s_{i}}(x_{i})$. If $\phi^{s_{i}}(x_{i})$ itself is also linear respect to $x_{i}$, then only the linear functions can be parameterized by MPS. To introduce nonlinearity into the model, one idea is to use the nonlinear kernel function $\phi^{s_{i}}(x_{i})$\footnote{For example, a frequently used kernel $\phi^{s_{i}}(x^{i}) = [\sin{(x_{i})}, \cos{(x_{i})}]$, see \citep{stoudenmire2016supervised}}, another idea is to use the nonlinear activation function $\sigma(\cdot)$ to make a nonlinear transformation of the respond $\Psi^{l}(\mathbf{x}|\mathbf{w}, \mathcal{B})$ to get an appropriate decision function \citep{efthymiou2019tensornetwork}. For the first idea, since the kernel functions are usually fixed, the 'nonlinearity' injected by the kernel function cannot be adapted according to the observations which reduced the flexibility of the model. For the second idea, the ranges of the frequently used activation functions cannot cover $\mathbb{R}$ ($\mathbb{R}^{k}$), the activated MPS cannot be a good candidate for the universal approximators.

In our discussions, we propose the following modified MPS structure, \begin{align}
\label{eq: modified_mps} \Psi(\mathbf{x}|\mathbf{w}, \mathcal{B}) =\sum_{l}\sigma(\sum_{\{\alpha, s\}}A^{s_{1}}_{\alpha_{1}\alpha_{2}}\cdots A^{ls_{i}}_{\alpha_{i}\alpha_{i+1}}\cdots A^{s_{n}}_{\alpha_{n}\alpha_{1}}\Phi^{s_{1}\cdots s_{n}}(\mathbf{x})),
\end{align}
as the universal approximator of the continuous functions on $\mathbb{R}$. Here $\sigma(\cdot)$ represents the scale-invariant sigmoidal function defined as 
\begin{align}
\label{eq: sigmoid}
   \sigma(x) \to 
   \begin{cases}
   0 &x \to -\infty\\
   C &x \to +\infty
   \end{cases},
\end{align}
where $C$ can be arbitrary real number in $\mathbb{R}$. Another requirement of the $\sigma(\cdot)$ is scaling invariance: scaling $\sigma(\cdot)$ with an arbitrary real number will not change its function form. The reason for this point is that with rescalable activation function $\sigma(\cdot)$\footnote{Two examples scale-invariant sigmoidal function are $\sigma(\cdot) = \frac{1}{C + \exp{(\cdot})}$ or $\frac{C}{1 + \exp{(\cdot)}}$.}, MPS \eqref{eq: modified_mps} is 'linear', or more strictly, the function space of MPS itself is a vector subspace of the continuous functions on $n$ dimensional unit cube denoted as $C^{0}[I^{n}]$. We will show this in section \ref{sec: linearity_of_mps}.
%The nice point of MPS is that arbitrary continuous functions can be approximated by \eqref{eq: modified_mps}, however hidden layers are necessary in neural-networks approximator of continuous functions. 

As we show that arbitrary boolean functions can be accurately implemented by the normal MPS \eqref{eq: pure_mps} and any continuous functions can be approximated to arbitrarily well by activated MPS \eqref{eq: modified_mps}, a natural question interests us is among the kernel dimension $|s_{i}|$, the bond dimension $|\alpha_{i}$ and the number of tensor nodes $n$, which one is the most important structure in MPS? It is well-known that for neural networks, the hidden layers are the key structure and the idea of deep learning is to increase the number of hidden layers of the neural networks model to extract high level 'features' of the data \cite{lecun2015deep}. For MPS, however, we argue that the kernel function $\Phi^{s_{1} \cdots s_{n}}$ \eqref{eq: kernel_function} which determines the exponentially large 'feature' space are essential. Thanks for the mapping from raw $n$ dimensional data space $\mathbf{x}$ to the $d^{n}$ dimensional features space introduced by the kernel function, MPS \eqref{eq: pure_mps} can represent any boolean functions however at least one hidden layer is needed if neural networks are used. Actually in section \ref{sec: relation_mps_neural_nets} we will show that MPS with activation function \eqref{eq: modified_mps} is equivalent to one hidden layer neural networks equipped by a kernel function defined by $\Phi^{s_{1}\cdots s_{n}}(\cdot)$.
%MPS are powerful tensor networks to capture the one dimensional correlation and widely used in machine learning recently. Actually, MPS is an approximation of an order-N tensor by a contracted chain of N lower order tensors.
%\begin{align*}
%    &\Psi^{l}(\bold{x}^{(i)}) = f(\psi^{l}(\bold{x}^{(i)}))\\
%    &\psi^{l}(\bold{x}^{(i)}) =A^{s_{1}}_{\alpha_{1}}A^{s_{2}}_{\alpha_{1}\alpha_{2}}\cdots A^{s_{i}l}_{\alpha_{i-1}\alpha_{i}}\cdots A^{s_{n-1}}_{\alpha_{n-2}\alpha_{n-1}} A^{s_{n}}_{\alpha_{n-1}}\Phi^{s_{1}s_{2}\cdots s_{n}}(\bold{x}^{(i)})\\
%    &\Phi^{s_{1}s_{2}\cdots s_{n}}(\bold{x}^{(i)}) = \phi^{s_{1}}(x^{(i)}_{1})\otimes\phi^{s_{2}}(x^{(i)}_{2})\cdots\otimes\phi^{s_{n}}(x^{(i)}_{n})\\
%    &\bold{x}^{(i)} = (x^{(i)}_{1},\cdots, x^{(i)}_{n})
%\end{align*}
%In the above formulas, $\bold{x}^{(i)}$ represent the $i'th$ sample in the dataset $\mathrm{D} = \{\bold{x}^{(i)}| i=1, 2, 3,\cdots m\}$, and $f$ represent the active function. $\psi^{l}$ represent the matrix product state and $A^{s_{i}}_{\alpha_{i-1}\alpha_{i}}$ represent one tensor block of the MPS. $\otimes$ means the tensor product and $\phi^{s_{1}}$ is a map from $\mathbb{R}$ to $\mathbb{R}^{|s_{1}|}$, which can be understood as the kernel function.
%
%We note here that in the MPS $\psi^{l}$, there are about $nd^{2}$
%%\begin{align*}
%%    2(n-2)d^{2} + 4d
%%\end{align*}
%parameters, where $n$ means the number of tensors and $d$ means the bond dimension.

%!TEX root = main.tex

\section{Representation of Arbitrary Boolean Functions}

%To use the Tensor Networks as a learning machine, naturally we will get interested in the hypothesis space the Tensor Networks can represent. For Neural Networks, a lot of work has been done. It has already been proved that any continuous functions can be approximated to arbitrary accuracy by the neural networks with only one hidden layer\citep{cybenko1989approximation, stinchcombe1990approximating, hornik1991approximation}. To be specific, it means that the function space defined by one hidden layer neural networks is dense in the $C^{0}$ space. Some further work also improved the universal approximation theorem by relaxing the requirement of the activation function.
%
%Another interesting direction to explore the representation power of the Neural Networks is to study arbitrary boolean function approximation. Boolean function is defined as 
Boolean functions play a central role in propositional logic, complexity theory, computation theory, etc. In \cite{biamonte2011categorical}, the authors propose to represent the boolean gates by encoding the truth table of a given boolean gate into the quantum states. For an arbitrary boolean function $f$,
\begin{align}
\label{eq: boolean_fun}
    f: \{0, 1\}^{n} \to \{0, 1\}, \quad n\in \mathrm{Z}, 
\end{align}
the quantum state $\ket{f}$ encoding this boolean gate $f$ is 
\begin{align}
\label{eq: boolean_state}
\ket{f} := \sum_{\{X_{i}\}}\ket{X_{1}}\otimes\cdots\otimes\ket{X_{n}}\otimes\ket{f(X_{1}, \cdots, X_{n})}.
\end{align}
%and it  As far as we know, the study of representation of the boolean function by Neural Network originates from research of using neural networks to simulate the finite state machine.
%We consider the arbitraryIt is well known now that one hidden layer perceptron can represent any boolean function. In this work, we proved that just one layer of tensor network, arbitrary boolean function can be represented without needing multi-layers blocks.
Different from the idea in \eqref{eq: boolean_state}, we use the MPS, namely the contraction of a series of low rank tensors to represent arbitrary boolean gate. Moreover, a specific basis is chosen in our method and then a quantum state $\ket{\Psi}$ is noting more than a vector and the quantum gate is just a matrix or a higher order tensor. We note here that in our case we do not distinguish the upper or lower indices since the vector space is defined on the real field and also it is isomorphic to its dual. 

In following sections, we will firstly show that how to represent the basic boolean gate: \textit{And}, \textit{Or} and \textit{Not} in \ref{sec: basic_boolean}. Then we will show how to use the MPS to implement \textit{universal And} gate. Since the truth table of a given boolean gate can be encoded into the \textit{universal And} gate, an arbitrary boolean function can be realized by the \textit{unverisal And} gate and then represented by MPS in \ref{sec: arbitrary_boolean}. With the same idea, we discuss the implementation more general n-input boolean gate such as parity function and threshold function in \ref{sec: n_input_boolean}.
%the output of the boolean gate is not a variable of the representation function. 

\subsection{Basic Boolean Functions}
\label{sec: basic_boolean}

We will show that the MPS can represent the basic logic operations and the secondary logic operations.

\begin{theorem}
\label{the: and_gate}
MPS can represent \textit{AND} gate. More specifically. For $X_{1}, X_{2} \in \{0, 1\}$, there exists a MPS $\Psi$ \eqref{eq: pure_mps} such that 
\begin{align}
    \Psi(X_{1}, X_{2}) = X_{1}\wedge X_{2}.
\end{align}
\end{theorem}

\begin{proof}
We can set the kernel function $\Phi(X_{1}, X_{2})$ to be,
\begin{align}
\label{eq: boolean_kernel_and}
    &\Phi(X_{1}, X_{2}) = \phi_{1}(X_{1})\otimes\phi_{2}(X_{2}),\\ 
    &\phi_{i}(X_{i}) = [X_{i}, 1 - X_{i}].
\end{align}
To represent the And gate, each tensor node $A^{s_{i}}_{\alpha_{i}\alpha_{i+1}}$ with bond dimension $|\alpha| = 1$ is enough,
\begin{align}
    A^{s_{i}}_{\alpha_{i}\alpha_{i+1}} = A^{s_{i}} = 
    [1, 0] 
\end{align}
Above MPS can implement the \textit{And} gate.
\end{proof}

\begin{theorem}
\label{the: or_gate}
MPS can represent \textit{OR} gate. More specifically. For $X_{1}, X_{2} \in \{0, 1\}$, there exists a MPS $\Psi$ \eqref{eq: pure_mps} such that 
\begin{align}
    \Psi(X_{1}, X_{2}) = X_{1}\vee X_{2}.
\end{align}
\end{theorem}

\begin{proof}
We can set the kernel function $\Phi(X_{1}, X_{2})$ to be,
\begin{align}
\label{eq: boolean_kernel_or}
    &\Phi(X_{1}, X_{2}) = \phi_{1}(X_{1})\otimes\phi_{2}(X_{2}),\\ 
    &\phi_{i}(X_{i}) = [X_{i}, 1 - X_{i}].
\end{align}
To represent the \textit{OR} gate, we should set the bond dimension $|\alpha|$ of each tensor node $A^{s_{i}}_{\alpha_{i}\alpha_{i+1}}$ to be $3$,
\begin{align}
    &A^{s_{1}}_{\alpha_{1}\alpha_{2}} = 
    [\begin{bmatrix}
    1, & 0, & 1
    \end{bmatrix},
    \begin{bmatrix}
    0, & 1, & 0 
    \end{bmatrix}], \\
    &A^{s_{2}}_{\alpha_{2}\alpha_{1}} = 
    [\begin{bmatrix}
    0, & 1, & 1 
    \end{bmatrix},
    \begin{bmatrix}
    1, & 0, & 0 
    \end{bmatrix}].
\end{align}
Above MPS can implement the \textit{OR} gate.
\end{proof}

\begin{theorem}
\label{the: not_gate}
MPS can represent \textit{Not} gate. More specifically. For $X_{1} \in \{0, 1\}$, there exists a MPS $\Psi$ \eqref{eq: pure_mps} such that 
\begin{align}
    \Psi(X_{1}) = \neg X_{1}.
\end{align}
\end{theorem}
\begin{proof}
We can set the kernel function $\Phi(X_{1})$ to be,
\begin{align}
\label{eq: boolean_kernel_not}
    &\Phi(X_{1}) = \phi_{1}(X_{1}),\\ 
    &\phi_{1}(X_{1}) = 1 - X_{1},
\end{align}
in which $|s_{i}| = 1$.
To represent the \textit{Not} gate, each tensor node $A^{s_{i}}_{\alpha_{i}\alpha_{i+1}}$ with bond dimension $|\alpha| = 1$ is enough,
\begin{align}
    A^{s_{1}}_{\alpha_{1}\alpha_{2}} = A^{s_{1}} = 1 
\end{align}
Above MPS can implement the \textit{Not} gate.
\end{proof}

\subsection{Universal \textit{And}/\textit{OR} Gate and Arbitrary Boolean Functions}
\label{sec: arbitrary_boolean}

In this section we will construct the MPS for the \textit{Universal And} gate as follows. With the \textit{Universal And} gate, we can use the MPS to represent arbitrary boolean functions.

\begin{theorem}
\label{the: and_represent}
MPS can represent universal \textit{AND} gate. More specifically. For $X_{1}$, $X_{2}$, $\cdots$, $X_{n} \in \{0, 1\}$, there exists a MPS $\Psi$ \eqref{eq: pure_mps} such that 
\begin{align}
    \Psi(X_{1}, \cdots, X_{n}) = (\wedge_{i=1}^{l}X_{i})\wedge(\wedge_{j=l+1}^{n}\bar{X}_{j}).
\end{align}

\begin{proof}
We can set the kernel function $\Phi(X_{1}, \cdots, X_{n})$ to be,
\begin{align}
\label{eq: boolean_kernel_universal_and}
    &\Phi(X_{1}, \cdots, X_{n}) = \phi_{i}(X_{i})^{\otimes n},\\ 
    &\phi_{i}(X_{i}) = [X_{i}, 1 - X_{i}].
\end{align}
To represent the And gate, each tensor node $A^{s_{i}}_{\alpha_{i}\alpha_{i+1}}$ with bond dimension $|\alpha| = 1$ is enough,
\begin{align}
    A^{s_{i}}_{\alpha_{i}\alpha_{i+1}} = A^{s_{i}} = 
    \begin{cases}
    [1, 0] \quad \mbox{if $X_{i}$} \\ 
    [0, 1] \quad \mbox{if $\bar{X}_{i}$} 
    \end{cases}.
\end{align}
So for any And gate, we can construct a MPS to implement it.
\end{proof}
%where 
%\begin{align}
%   \Psi(X_{1}, \cdots, X_{n}) = \sum_{\{s, \alpha\}}A^{s_{1}}_{\alpha_{1}}A^{s_{2}}_{\alpha_{1}\alpha_{2}}\cdots A^{s_{n-1}}_{\alpha_{n-2}\alpha_{n-1}}A^{s_{n}}_{\alpha_{n}}\Phi^{s_{1}\cdots s_{n}}(X_{1}, \cdots, X_{n})
%\end{align}
\end{theorem}

\begin{theorem}
\label{the: arbitrary_implementation}
\textit{An arbitrary boolean function} $f$ \eqref{eq: boolean_fun} can be exactly represented by MPS. 
\end{theorem}

\begin{proof}
\label{pf: pf_of_theorem_2}

The kernel function used here is the same as \eqref{eq: boolean_kernel_universal_and}. For arbitrary boolean gates, we can write down the corresponding boolen expression as the summation of several \textit{universal And} gates \eqref{eq: and_gate_exp} by the truth table , namely
\begin{align}
\label{eq: truth_table_exp}
    f(X_{1}, \cdots, X_{n}) = \sum_{\{X_{1}, \cdots, X_{n}\}}\bar{X}_{1}\wedge\cdots\wedge X_{n}.
\end{align}
We will set the bond dimension $|\alpha_{i}|$ to be the number of the terms in \eqref{eq: truth_table_exp} which is denoted by $m$, and all the components of $A^{s_{i}}_{\alpha_{i}\alpha_{j}}$ are filled with $0$ or $1$ according the following rules. 

For the boundary tensors, we impose the open-bounday condition and the two bounday tensors are $|\alpha|$ dimensional vectors as follows, 
\begin{align}
    &A^{s_{1}}_{\alpha_{1}} = [
    \begin{bmatrix}
    X_{1}^{(1)}, & \cdots, & 
    X_{1}^{(m)}
    \end{bmatrix}, 
    \begin{bmatrix}
    \bar{X}_{1}^{(1)} & \cdots, & \bar{X}_{1}^{(m)}
    \end{bmatrix}
    ],\\
    &A^{s_{1}}_{\alpha_{1}} = [
    \begin{bmatrix}
    X_{n}^{(1)}, & \cdots, & 
    X_{n}^{(m)}
    \end{bmatrix}, 
    \begin{bmatrix}
    \bar{X}_{n}^{(1)} & \cdots, & \bar{X}_{n}^{(m)}
    \end{bmatrix}
    ].
\end{align}
For the other three-order tensors, we set the tensors $A^{s_{i}}_{\alpha_{i}\alpha_{j}}$ to be  the form as 
\begin{align}
\label{eq: filling_rule}
 A^{s_{i}}_{\alpha_{i}\alpha_{i+1}} = [
 \oplus_{j=1}^{m} X_{i}^{(j)},
 \oplus_{j=1}^{m} \bar{X}_{i}^{(j)}
 ].
\end{align}
%Where,
%\begin{align}
%    &A^{s_{j}} = [1, 1],\\
%    &C_{\alpha_{j}} = 
%    \begin{cases}
%    x \\
%    y
%    \end{cases}
%\end{align}
In above formula, the dimension of kernel $|s_{i}|$ is $2$. Each component of dimension $s_{i}$ is a $m$ dimensional diagonal matrix. 
%which means that the tensors in the middle of the MPS are diagonal matrix for each specific implementation of the upper index.
%For the $\mathbf{i}$'th term in the expression, we will focus on the i'th component of each tensor nodes. 

More specifically, if $j$'th boolean variable in the $i$'th term of expression \eqref{eq: truth_table_exp} is $X_{j}$, then we set the $A^{1}_{ii}$ component of $j$'th tensor to be $1$. Otherwise if it is $\bar{X}_{j}$, then we set $A^{2}_{ii}$ to be $1$. We iterate above process for all $m$ terms.
%After one term is finished, we know one component of each 
With above rules, we know that only one component of $s_{i}$ for fixed $\alpha_{i}$ of $A^{s_{i}}_{\alpha_{i}\alpha_{j}}$ is set to be $1$. 
%We iterate this step for all terms in the expression. The remaining components of the MPS are set to be $0$. 
The MPS we get is just the representation of the boolean gate given in \eqref{eq: truth_table_exp}.
\end{proof}

\begin{corollary}
MPS can represent universal \textit{OR} gate. More specifically. For $X_{1}$, $X_{2}$, $\cdots$, $X_{n} \in \{0, 1\}$, there exists a MPS $\Psi$ \eqref{eq: pure_mps} such that 
\begin{align}
    \Psi(X_{1}, \cdots, X_{n}) = (\vee_{i=1}^{l}X_{i})\vee(\vee_{j=l+1}^{n}\bar{X}_{j}).
\end{align}
\end{corollary}
%\begin{proof}
%We can write down the OR gate as a sum of several And terms by the Truth table. Actually, we have 
%We prove the above theorem directly. Namely, we just need to construct the MPS $A^{s_{1}}_{\alpha_{1}}$, $B^{s_{2}}_{\alpha_{2}}$ and the kernel functions $\phi^{s_{1}}$ and $\phi^{s_{2}}$ to show the equivalence holds. 
%
%Here we set $|s_{1}| = 2$, $|\alpha_{1}| = 2$, 
%\begin{align*}
%    &\phi^{s_{1}}(x_{i}) = [x_{i}, 1-x_{i}],\\
%    &A^{s_{1}}_{\alpha_{1}} = [[2w, w], [-w, w]],\\
%    &B^{s_{2}}_{\alpha_{2}} = [[w, 2w], [w, -w]],\\
%    & w\in \mathbb{R}
%\end{align*}
%\end{proof}

In following example, we will demonstrate the MPS for a $3$-Input OR boolean gate.

\begin{example} [$3$-Input OR gate]
We consider the $3$-Input OR gate: $f(X_{1}, X_{2}, X_{3}) = X_{1}\vee X_{2}\vee X_{3}$ with truth table as \ref{table: or_table}. We write down the summation form of the above boolean table as 
\begin{align}
\label{eq: and_gate_exp}
    f(X_{1}, X_{2}, X_{3}) = & X_{1}\wedge\bar{X}_{2}\wedge\bar{X}_{3} + \bar{X}_{1}\wedge X_{2}\wedge\bar{X}_{3} + \bar{X}_{1}\wedge \bar{X}_{2}\wedge X_{3} + \bar{X}_{1}\wedge X_{2}\wedge X_{3}\\ 
    & + X_{1}\wedge\bar{X}_{2}\wedge X_{3} + X_{1}\wedge X_{2}\wedge\bar{X}_{3} + X_{1}\wedge X_{2}\wedge X_{3}.
\end{align}
Following the 'construction rule' introduced in the proof of Theorem \ref{the: arbitrary_implementation}, we can get the MPS of above expression \eqref{eq: and_gate_exp} as 
\begin{align}
\label{eq: or_example}
    &A^{s_{1}}_{\alpha_{1}\alpha_{2}} = [[1, 0, 0, 0, 1, 1, 1], [0, 1, 1, 1, 0, 0, 0]], \\
    &A^{s_{2}}_{\alpha_{2}\alpha_{3}} = [
    \begin{bmatrix}
    0 & & & & & & \\ 
    & 1 & & & & & \\
    & & 0 & & & & \\
    & & & 1 & & & \\
    & & & & 0 & & \\
    & & & & & 1 & \\
    & & & & & & 1
    \end{bmatrix},
    \begin{bmatrix}
    1 & & & & & & \\
    & 0 & & & & & \\
    & & 1 & & & & \\ 
    & & & 0 & & & \\
    & & & & 1 & & \\
    & & & & & 0 & \\
    & & & & & & 0
    \end{bmatrix}
    ], \\
    &A^{s_{3}}_{\alpha_{3}\alpha_{1}} = [[0, 0, 1, 1, 1, 0, 1], [1, 1, 0, 0, 0, 1, 0]].
\end{align}
\end{example}

\begin{table}[!t!h]
\centering
\begin{tabular}{|c|c|c||c|}
\hline
\multicolumn{3}{|c||}{Input} & Output \\ \hline
$X_{1}$ & $X_{2}$ & $X_{3}$ & $f(X_{1}, X_{2}, X_{3})$ \\ \hline
0 & 0 & 0 & 0 \\ \hline
1 & 0 & 0 & 1 \\ \hline
0 & 1 & 0 & 1 \\ \hline
0 & 0 & 1 & 1 \\ \hline
0 & 1 & 1 & 1 \\ \hline
1 & 0 & 1 & 1 \\ \hline
1 & 1 & 0 & 1 \\ \hline
1 & 1 & 1 & 1 \\ \hline
\end{tabular}
\caption{The truth table of the OR gate: $X_{1}\vee X_{2}\vee X_{3}$.}
\label{table: or_table}
\end{table}
%\begin{theorem}
%MPS can represent arbitrary boolean function. 
%\end{theorem}
%\begin{proof}[Proof of Theorem 3]
%
%We prove the above theorem directly. Namely, we just need to construct the MPS $A^{s_{1}}_{\alpha_{1}}$ and the kernel functions $\phi^{s_{1}}$ to show the equivalence holds. 
%
%Here we set $|s_{1}| = 2$, $|\alpha_{1}| = 2$, 
%\begin{align*}
%    &\phi^{s_{1}}(x_{i}) = [x_{i}, 1-x_{i}],\\
%    &A^{s_{1}}_{\alpha_{1}} = [-e^{w_{1}}, e^{w_{2}}],\\
%    &w_{1}, w_{2} \in \mathbb{R}
%\end{align*}
%\end{proof}
%
%\begin{theorem}
%$\forall x_{1}, x_{2} \in \{0, 1\}, \exists\phi^{s_{1}},\phi^{s_{2}}, 
%A^{s_{1}}_{\alpha_{1}}, B^{s_{2}}_{\alpha_{2}}, \textit{such that } x_{1}\oplus x_{2} = \sum_{\{s_{1}, s_{2}, \alpha_{1}\}}A^{s_{1}}_{\alpha_{1}}B^{s_{1}}_{\alpha_{1}}\phi^{s_{1}}(x_{1})\phi^{s_{2}}(x_{2})$
%\end{theorem}
%
%\begin{proof}[Proof of Theorem 4]
%
%We prove the above theorem directly. Namely, we just need to construct the MPS $A^{s_{1}}_{\alpha_{1}}$, $B^{s_{2}}_{\alpha_{2}}$ and the kernel functions $\phi^{s_{1}}$ and $\phi^{s_{2}}$ to show the equivalence holds. 
%
%Here we set $|s_{1}| = 2$, $|\alpha_{1}| = 2$, 
%\begin{align*}
%    &\phi^{s_{1}}(x_{i}) = [x_{i}, 1-x_{i}],\\
%    &A^{s_{1}}_{\alpha_{1}} = [[-w, -w], [w, w]],\\
%    &B^{s_{2}}_{\alpha_{2}} = [[w, w], [-w, -w]],\\
%    & w\in\mathbb{R}.
%\end{align*}
%\end{proof}
%
%From above theorems, we found that there is s scaling symmetry in the tensor networks.

\subsection{n-Input Boolean Gates}
\label{sec: n_input_boolean}

In boolean algebra, the parity function is a boolean function whose value is $1$ if and only if the input vector $[X_{1}, \cdots, X_{n}]$ contains an odd number of ones\footnote{The $2$-Input Parity gate is the \textit{XOR} gate.}:
\begin{align}
\label{eq: parity_fun}
  f(X_{1}, \cdots, X_{n}) = 
  \begin{cases} 1 &\mbox{if number of $1$ is odd} \\ 
  0 & \mbox{if number of $1$ is even} \end{cases}.
\end{align}

\begin{corollary}
MPS can represent universal \textit{Parity} gate \eqref{eq: parity_fun}.
\end{corollary}

\begin{example} [$3$-Input Parity Function]
\label{exp: parity_fun}
We consider the $3$-Input Parity gate with the truth table as \ref{table: parity_table}. The boolean expression is as
\begin{align}
\label{eq: 3_input_parity_exp}
    f(X_{1}, X_{2}, X_{3}) = & X_{1}\wedge\bar{X}_{2}\wedge\bar{X}_{3} + \bar{X}_{1}\wedge X_{2}\wedge\bar{X}_{3} + \bar{X}_{1}\wedge \bar{X}_{2}\wedge X_{3} + X_{1}\wedge X_{2}\wedge X_{3}.
\end{align}
Then we can write down the MPS for the $3$-input Parity function as 
\begin{align}
\label{eq: parity_example}
    &A^{s_{1}}_{\alpha_{1}\alpha_{2}} = [[1, 0, 0, 1], [0, 1, 1, 0]], \\
    &A^{s_{2}}_{\alpha_{2}\alpha_{3}} = [
    \begin{bmatrix}
    0 & & &\\ 
    & 1 & & \\
    & & 0 &\\
    & & & 1
    \end{bmatrix},
    \begin{bmatrix}
    1 & & &\\
    & 0 & &\\
    & & 1 &\\ 
    & & & 0\\
    \end{bmatrix}
    ], \\
    &A^{s_{3}}_{\alpha_{3}\alpha_{1}} = [[0, 0, 1, 1], [1, 1, 0, 0]].
\end{align}
\end{example}

\begin{table}[!t!h]
\centering
\begin{tabular}{|c|c|c||c|}
\hline
\multicolumn{3}{|c||}{Input} & Output \\ \hline
$X_{1}$ & $X_{2}$ & $X_{3}$ & $f(X_{1}, X_{2}, X_{3})$ \\ \hline
0 & 0 & 0 & 0 \\ \hline
1 & 0 & 0 & 1 \\ \hline
0 & 1 & 0 & 1 \\ \hline
0 & 0 & 1 & 1 \\ \hline
0 & 1 & 1 & 0 \\ \hline
1 & 0 & 1 & 0 \\ \hline
1 & 1 & 0 & 0 \\ \hline
1 & 1 & 1 & 1 \\ \hline
\end{tabular}
\caption{The truth table of $3$-Input \textit{Parity} Gate \eqref{eq: 3_input_parity_exp}}
\label{table: parity_table}
\end{table}
%\subsection{Parity Function}
%
%\begin{theorem}
%$\forall x_{1}, x_{2}, \cdots x_{n} \in \{0, 1\}, \exists\phi^{s_{1}},\cdots \phi^{s_{n}}, 
%A^{s_{1}}_{\alpha_{1}},\cdots A^{s_{i}}_{\alpha_{i-1}\alpha_{i}}\cdots,A^{s_{n}}_{\alpha_{n-1}} \textit{such that } f(x_{1}, x_{2},\cdots x_{n}) = -H(\sum_{\{s_{1}, s_{2},\cdots,s_{n}, \alpha_{1}, \alpha_{2},\cdots, \alpha_{n-1}\}}A^{s_{1}}_{\alpha_{1}},\cdots A^{s_{i}}_{\alpha_{i-1}\alpha_{i}}\cdots,A^{s_{n}}_{\alpha_{n-1}}\phi^{s_{1}}(x_{1})\cdots\phi^{s_{n}}(x_{n}))$
%\end{theorem}
%
%\begin{proof}[Proof of Theorem 5]
%We prove the above theorem directly. Namely, we just need to construct the MPS $A^{s_{1}}_{\alpha_{1}}$, $B^{s_{2}}_{\alpha_{2}}$ and the kernel functions $\phi^{s_{1}}$ and $\phi^{s_{2}}$ to show the equivalence holds. 
%
%Here we set $|s_{1}| = 2$, $|\alpha_{1}| = 1$, 
%\begin{align*}
%    &\phi^{s_{1}}(x_{i}) = [x_{i}, 1-x_{i}],\\
%    &A^{s_{i}}_{\alpha_{i-1}\alpha_{i}} = [-e^{w_{i-1}}, e^{w_{i}}],\\
%    &w_{i}\in\mathbb{R}, i\in \{1, 2, \cdots, n\}.
%\end{align*}
%\end{proof}
%
%From above theorem, we found that there is a permutation symmetry in the tensor networks.
%

A $n$-Input Threshold gate $Th^{n}_{k}$ is activated if $k$ or more than $k$ components of the input binary vector are $1$. Similar to the universal gates we discussed before, Threshold gates are symmetric boolean function as follows,
\begin{align}
\label{eq: threshold_fun}
\forall S \in \{0, 1\}^{n}, 
\quad
Th^{n}_{k}(S) = 
\begin{cases}
1 &\mbox{if $\sum S_{i}\geq k$} \\
0 &\mbox{if $\sum S_{i}<k$}
\end{cases}.  
\end{align} 

\begin{corollary} 
MPS can represent universal \textit{Threshold} function \eqref{eq: threshold_fun}.
\end{corollary}

\begin{example} [3-Input Threshold Gate $Th_{2}^{3}$] We use the threshold function $Th^{2}_{3}$ whose truth table is \ref{table: threshold_table} as an example. The boolean expression for $Th^{2}_{3}$ is 
\begin{align}
\label{eq: threshold_exp}
    Th_{2}^{3}(X_{1}, X_{2}, X_{3}) = & X_{1}\wedge X_{2}\wedge\bar{X}_{3} + X_{1}\wedge \bar{X}_{2}\wedge X_{3} + \bar{X}_{1}\wedge X_{2}\wedge X_{3} + X_{1}\wedge X_{2}\wedge X_{3}.
\end{align}
So we can get the MPS of $The^{3}_{2}$ is as 
\begin{align}
\label{eq: 3_input_the}
    &A^{s_{1}}_{\alpha_{1}\alpha_{2}} = [[1, 1, 0, 1], [0, 0, 1, 0]], \\
    &A^{s_{2}}_{\alpha_{2}\alpha_{3}} = [
    \begin{bmatrix}
    1 & & &\\ 
    & 0 & & \\
    & & 1 &\\
    & & & 1
    \end{bmatrix},
    \begin{bmatrix}
    0 & & &\\
    & 1 & &\\
    & & 0 &\\ 
    & & & 0\\
    \end{bmatrix}
    ], \\
    &A^{s_{3}}_{\alpha_{3}\alpha_{1}} = [[0, 1, 1, 1], [1, 0, 0, 0]].
\end{align}
\end{example}

\begin{table}[!t!h]
\centering
\begin{tabular}{|c|c|c||c|}
\hline
\multicolumn{3}{|c||}{Input} & Output \\ \hline
$X_{1}$ & $X_{2}$ & $X_{3}$ & $f(X_{1}, X_{2}, X_{3})$ \\ \hline
0 & 0 & 0 & 0 \\ \hline
1 & 0 & 0 & 0 \\ \hline
0 & 1 & 0 & 0 \\ \hline
0 & 0 & 1 & 0 \\ \hline
0 & 1 & 1 & 1 \\ \hline
1 & 0 & 1 & 1 \\ \hline
1 & 1 & 0 & 1 \\ \hline
1 & 1 & 1 & 1 \\ \hline
\end{tabular}
\caption{The truth table of \textit{Threshold} Gate: $Th_{2}^{3}$ \eqref{eq: 3_input_the}.}
\label{table: threshold_table}
\end{table}
%\subsection{Threshold Functions}
%In this section, we will study the Boolean function. Actually, a Boolean function is a map $f: \{0, 1\}^{n}\to \{0, 1\}$. For every Boolean function, it can be calculated by Boolean tree. We can also write the Boolean functions as maps from $2^{[n]}\to\{0, 1\}$, where $2^{[n]}$ denotes the power set of $[n]$.
%
%The Threshold function $Th^{n}_{k}: 2^{[n]}\to \{0, 1\}$ is defined as follows,
%\subsection{Exponential Increasing Rate of the Number of Tensors}

\subsection{Complexity of MPS and Improving Efficiency by Karnaugh Map}

%Using Karnaugh Map to Simplify the MPS and its efficiency}

For a $n$-input boolean gate, we know that in the extreme case, by our method \ref{pf: pf_of_theorem_2}, the bond dimension $|\alpha_{i}|$ should be set at the order of $O(2^{n})$ which is very space expensive.  

It is easy to extend our method using the Karnaugh Map to simplify the boolean expression into the reduced Disjunctive Normal Form (DSF). With the K-Map, the number of terms in the reduced DNF is of order $O(2^{n-1})$ and the order of the number of parameters in the corresponding MPS is $O(n2^{n-1})$ which implies that the MPS and one-hidden-layer neural networks are equally efficient in realizing the boolean functions. Actually this is reasonable since MPS is equivalent to the neural networks equipped with kernel functions. Roughly the kernel function $\Phi^{s_{1}\cdots s_{n}}$ plays the role of hidden layer in MPS. In section \ref{sec: relation_mps_neural_nets}, we will discuss these points in more detail.

%\begin{theorem}
%$\exists \phi^{s_{1}},\cdots \phi^{s_{n}}, A^{s_{1}}_{\alpha_{1}},\cdots A^{s_{i}}_{\alpha_{i-1}\alpha_{i}}\cdots, A^{s_{n}}_{(n)\alpha_{n-1}}, 
%\forall S\in 2^{[n]}, S = \{x_{1}, x_{2}, \cdots, x_{n-1}, x_{n}\}, \textit{such that}$\\$f(S) = Th(S)$.\\ $f(S) := H(\log{(\Psi(S)/G)})$, \textit{where} 
%\begin{align*} 
%\Psi(S) = \sum_{\{s_{1}, s_{2},\cdots,s_{n}, \alpha_{1}, \alpha_{2},\cdots, \alpha_{n-1}\}}A^{s_{1}}_{\alpha_{1}},\cdots A^{s_{i}}_{\alpha_{i-1}\alpha_{i}}\cdots,A^{s_{n}}_{\alpha_{n-1}}\phi^{s_{1}}(x_{1})\cdots\phi^{s_{n}}(x_{n}) 
%\end{align*}
%\textit{ and } $G$ \textit{ is determined by the parameters of the tensor networks }$\Psi$.
%\end{theorem} 
%
%\begin{proof}[Proof of Theorem 6]  Here we set $|s_{1}| = 2$, $|\alpha_{1}| = 1$,  \begin{align*} &\phi^{s_{1}}(x_{i}) = [x_{i}, 1-x_{i}],\\ &A^{s_{i}}_{\alpha_{i-1}\alpha_{i}} = [w_{i}, 1],\\ &w_{i}\in (1, +\infty), i\in \{1, 2, \cdots, n\}.\\ & G\in (G_{1}, G_{2}),  \textit{ where } G_{1} = max(G_{s_{k-1}}), G_{2} = min(G_{s_{k}})\\ & G_{1} = \{\prod_{i_{1}, i_{2}\cdots i_{k-1}}w_{j}\}, G_{2} = \{\prod_{i_{1}, i_{2}\cdots i_{k}}w_{j}\} \end{align*} \end{proof}
%\subsection{Tensor V.S. Neuron from Boolean function perspective}

%!TEX root = main.tex

\section{Universal Approximation of Continuous Functions} 

In this section, we study the MPS as an universal approximator of continuous functions. We already discussed the mildly modified MPS \eqref{eq: modified_mps} in \ref{sec: set-up}, and in this section we will show that the function space 'spanned' by the activated MPS with $n$-dimensional input is dense in the continuous function space on the $n$-dimensional cube.

\subsection{Main Results and Techniques Preparation}

We reformulate the universal approximator of continuous function property of MPS as following theorem,

\begin{theorem}
\label{the: main_the}
The MPS function space, namely the set of all the functions in the MPS form \eqref{eq: modified_mps_again}, 
\begin{align}
\label{eq: modified_mps_again}
\Psi(\mathbf{x}|\mathbf{w}, \mathcal{B}) = \sum_{l}\sigma(\sum_{\{\alpha, s\}}A^{s_{1}}_{\alpha_{1}\alpha_{2}}\cdots A^{ls_{i}}_{\alpha_{i}\alpha_{i+1}}\cdots A^{s_{n}}_{\alpha_{n}\alpha_{1}}\Phi^{s_{1}\cdots s_{n}}(\mathbf{x})),
\end{align}
is dense in $C^{0}(I^{n})$ with the supremum norm. More concretely, the continuous functions on the $n$-dimensional unit cube $I^{n}$ can be approximated by $\Psi(\mathbf{x})$ up to an arbitrarily small $\epsilon$ error in the sense of supremum norm, namely for any $f(x)\in C^{0}(I^{n})$ and any $\epsilon \in \mathbb{R}$, there exist a $\Psi(\mathbf{x})$ \eqref{eq: modified_mps_again} such that \begin{align}
    \sup_{\mathbf{x}}|\Psi(\mathbf{x}) - f(\mathbf{x})| < \epsilon.
\end{align}
\end{theorem}

The proof of above theorem \ref{the: main_the} is mainly based on the Hahn-Banach theorem and Riesz Representation theorem which are the fundamental results in the functional analysis, see \citep{folland2013real} for example. 

Similar to the ideas in \citep{cybenko1989approximation, hornik1991approximation}, we use the contradiction to show that closure of the function family of $\Psi(\mathbf{x})$ \eqref{eq: modified_mps_again}: $\bar{\mathcal{M}}$, is $C^{0}(I^{n})$. At first we show that $\mathcal{M}$ is a vector subspace of $C^{0}(I^{n})$. By Hahn-Banach theorem, there exists a non-vanishing functional (measure) which can be extended to the whole space $C^{0}(I^{n})$ on $\mathcal{M}$. However by Riesz Representation theorem and the property of the scale-invariant sigmoidal function, we can show that the measure on the $C^{0}(I^{n})$ always vanishes which contradicts with the fact we derived from the Hahn-Banach theorem. So $\mathcal{M}$ is dense in $C^{0}(I^{n})$. 

\subsection{Linearity of MPS}
\label{sec: linearity_of_mps}

In this section, we will show that $\Psi(\mathbf{x})$ function family $\mathcal{M}$ is a vector subspace of $C^{0}(I^{n})$ on real field $\mathbb{R}$. So we need to show that under $+$ operation, $M$ is a abelian group and also $M$ is closed under the $\times$ operation, namely the real scalar multiplication\footnote{The identity of $\mathcal{M}$ with $+$ as a commutative group is $0$. The associative law and distribution law is trivially satisfied.}.
%between real number and $\Psi(\mathbf{x})$ 

\begin{lemma}
\label{lemma: linearity_mps}
The space of the functions with the MPS form \eqref{eq: modified_mps}$: \mathcal{M}$, is a linear subspace of $C^{0}(I^{n})$.
\end{lemma}

\begin{proof}
We first show that $\mathcal{M}$ is closed under the $\times$ operation, 
\begin{align}
    \times: \mathbb{R} \times \mathcal{M} \to \mathcal{M}.
\end{align}
Since $\sigma(\cdot)$ \eqref{eq: rescaling_sigmoid} is invariant under scaling transformation, namely
\begin{align}
    k \cdot \sigma (A^{s_{1}}_{\alpha_{1}\alpha_{2}}\cdots A^{s_{n}}_{\alpha_{n}\alpha_{1}}\Phi^{s_{1} \cdots s_{n}}(\mathbf{x}))) = \sigma (\tilde{A}^{s_{1}}_{\alpha_{1}\alpha_{2}}\cdots \tilde{A}^{s_{n}}_{\alpha_{n}\alpha_{1}}\tilde{\Phi}^{s_{1} \cdots s_{n}}(\mathbf{x}))),
\end{align}
for an arbitrary $k \in \mathbb{R}$ and a function $\Psi(\mathbf{x}) \in \mathcal{M}$, there exists a $\tilde{\Psi}(\mathbf{x}) \in \mathcal{M}$ such that $k\Psi(\mathbf{x}) =  \Psi^{\prime}(\mathbf{x})$. 

Secondly, we will show that $\mathcal{M}$ is closed under $+$ operation, 
\begin{align}
+: \mathcal{M} \times \mathcal{M} \to \mathcal{M}.
\end{align}
For arbitrary two functions in $\mathcal{M}$, $\Psi_{1}(\mathbf{x})$ and $\Psi_{2}(\mathbf{x})$ as follows, 
\begin{align}
   &\Psi_{1}(\mathbf{x}) = \sum_{j}\sigma(\sum_{\{\beta, p\}}B^{p_{1}}_{\beta_{1}\beta_{2}}\cdots B^{jp_{i}}_{\beta_{i}\beta_{i+1}} \cdots B^{p_{n}}_{\beta_{n}\beta_{1}}\Phi^{p_{1} \cdots p_{n}}_{1}(\mathbf{x})), \\
   &\Psi_{2}(\mathbf{x}) = \sum_{k}\sigma(\sum_{\{\gamma, q\}}C^{q_{1}}_{\gamma_{1}\gamma_{2}}\cdots C^{kq_{i}}_{\gamma_{i}\gamma_{i+1}}\cdots C^{q_{n}}_{\gamma_{n}\gamma_{1}}\Phi^{q_{1} \cdots q_{n}}_{2}(\mathbf{x})),
   %&\Psi_{3}(\mathbf{x}) = \Psi_{1}(\mathbf{x}) + \Psi_{2}(\mathbf{x}), 
\end{align}
it is easy to verify that the sum of two MPS $\Psi_{3}(\mathbf{x}) = \Psi_{1}(\mathbf{x}) + \Psi_{2}(\mathbf{x})$ can be written in the MPS form as follows, 
\begin{align}
   &\Psi_{3}(\mathbf{x}) = \sum_{l}\sigma(\sum_{\{\alpha, s\}}A^{s_{1}}_{\alpha_{1}\alpha_{2}}\cdots A^{ls_{i}}_{\alpha_{i}\alpha_{i+1}}\cdots A^{s_{n}}_{\alpha_{n}\alpha_{1}}\Phi^{s_{1}, \cdots, s_{n}}(\mathbf{x})), \\
   &\Phi^{s_{1}\cdots s_{n}}(\mathbf{x}) = \phi^{s_{1}}(\mathbf{x}_{1})\otimes\cdots\otimes\phi^{s_{n}}(\mathbf{x}_{n}), 
\end{align}
where
\begin{align}
    &A^{s_{i}}_{\alpha_{i}\alpha_{i+1}} = 
    \begin{cases}
    B^{p_{i}}_{\beta_{i}\beta_{i+1}} \quad & \mbox{if $s_{i} = p_{i}$, $\alpha_{i} = \beta_{i}$, $\alpha_{i+1} = \beta_{i+1}$}, \\
    C^{q_{i}}_{\gamma_{i}\gamma_{i+1}} \quad & \mbox{if $s_{i} = q_{i}$, $\alpha_{i} = \gamma_{i}$, $\alpha_{i+1} = \gamma_{i+1}$}, \\
    0 \quad & \mbox{others},
    \end{cases}\\
    &A^{ls_{i}}_{\alpha_{i}\alpha_{i+1}} = 
    \begin{cases}
    B^{jp_{i}}_{\beta_{i}\beta_{i+1}} \quad & \mbox{if $l = j$, $s_{i} = p_{i}$, $\alpha_{i} = \beta_{i}$, $\alpha_{i+1} = \beta_{i+1}$}, \\
    C^{kq_{i}}_{\gamma_{i}\gamma_{i+1}} \quad & \mbox{if $l = k$, $s_{i} = q_{i}$, $\alpha_{i} = \gamma_{i}$, $\alpha_{i+1} = \gamma_{i+1}$}, \\
    0 \quad & \mbox{others},
    \end{cases}\\
    &\phi^{s_{i}}(\mathbf{x}_{i}) = 
    \begin{cases}
    \phi^{p_{i}}(\mathbf{x}_{i}) \quad \mbox{if $s_{i} = p_{i}$}, \\
    \phi^{q_{i}}(\mathbf{x}_{i}) \quad \mbox{if $s_{i} = q_{i}$}.
    \end{cases}
\end{align}
We note here that $|s| = |p| + |q|$, $|\alpha| = |\beta| + |\gamma|$ and $|l| = |j| + |k|$.

So above construction shows that the sum of two arbitrary MPS functions in $\mathcal{M}$ still lives in $\mathcal{M}$. Combining all above, we prove that $\mathcal{M}$ is a linear subspace.   
\end{proof}

\begin{example} [An example of sigmoidal function with scaling symmetry]
\label{exp: rescale_sigmoid_activation}
We consider an activation function as 
\begin{align}
\label{eq: rescaling_sigmoid}
    \sigma(\cdot) = \frac{1}{C + \exp{(\cdot})}.
\end{align}
where $C$ is a tunable parameter (variable) in $\mathbb{R}$.

So for the rescaling transformation of $\Psi(\mathbf{x})$, we have 
\begin{align}
   k \cdot \sigma(A^{s_{1}}_{\alpha_{1}\alpha_{2}} \cdots A^{s_{n}}_{\alpha_{n}\alpha_{1}}\Phi^{s_{1}\cdots s_{n}}(\mathbf{x})) &= \frac{k}{C + \exp{(A^{s_{1}}_{\alpha_{1}\alpha_{2}} \cdots A^{s_{n}}_{\alpha_{n}\alpha_{1}}\Phi^{s_{1}\cdots s_{n}}(\mathbf{x})})}\\
   & = \frac{1}{\frac{C}{k} + \frac{1}{k}\exp{(A^{s_{1}}_{\alpha_{1}\alpha_{2}} \cdots A^{s_{n}}_{\alpha_{n}\alpha_{1}}\Phi^{s_{1}\cdots s_{n}}(\mathbf{x}))}}\\
   & = \frac{1}{C^{\prime} + \exp{(\tilde{A}^{s_{1}}_{\alpha_{1}\alpha_{2}} \cdots \tilde{A}^{s_{n}}_{\alpha_{n}\alpha_{1}}\Phi^{s_{1}\cdots s_{n}}(\mathbf{x}))}}\\
   & = \sigma(\tilde{A}^{s_{1}}_{\alpha_{1}\alpha_{2}} \cdots \tilde{A}^{s_{n}}_{\alpha_{n}\alpha_{1}}\Phi^{s_{1}\cdots s_{n}}(\mathbf{x})).
\end{align}
With the activation function \eqref{eq: rescaling_sigmoid}, $\mathcal{M}$ is closed under scaling transformation. 
\end{example}

\subsection{Discrimination of the Scale-Invariant Sigmoidal Functions}

We will show the 'discrimination' property of the scale-invariant sigmoidal functions which is an essential property in the proof of our main theorem \ref{the: main_the}. Here 'discrimination' property means that if a finite signed measure $\mu$ annihilate an arbitrary scale-invariant sigmoidal function, then $\mu$ itself is trivial, namely $0$. However with Hahn-Banach theorem, we know that a non-vanishing finite signed measure which annihilate the scale-invariant sigmoidal function exists which contradicts the discrimination property. 

\begin{lemma}
\label{lemma: disc_lemma}
An arbitrary continuous scale-invariant sigmoidal function $\sigma(\cdot)$ \eqref{eq: sigmoid} is discriminatory. More specifically, for any $\sigma(\cdot)$ \eqref{eq: sigmoid} defined on the unit cube $I^{n}$, any finite, signed and regular Borel measure $\mu \in M(I^{n})$ and an arbitrary $l$ in $\{1, \cdots, D\}$, where $D$ is the dimension of the $l$ index, if
\begin{align}
\label{eq: sigmoid_integral}
   \int{\sigma(\sum_{\{\alpha, s\}}A^{s_{1}}_{\alpha_{1}\alpha_{2}}\cdots A^{s_{i}l}_{\alpha_{i}\alpha_{i+1}}\cdots A^{s_{n}}_{\alpha_{n}\alpha_{1}}\Phi^{s_{1}\cdots s_{n}}(\mathbf{x})) d\mu(\mathbf{x})} = 0, 
\end{align}
then $\mu(\mathbf{x}) = 0$.
\end{lemma}
\begin{proof}
%Without loss of generality, we can set parameter $C$ to be greater or equal to $0$.  
The main idea of the proof is to embed the scale-invariant sigmoidal function $\sigma(\cdot)$ into another auxiliary function $\tilde{\sigma}(\cdot)$ with another two parameters.  We introduce the $\tilde{\sigma}(\cdot)$ as 
\begin{align}
    \tilde{\sigma}(\mathbf{x}, \lambda, \eta) = \sigma(\lambda (\sum_{\{\alpha, s\}}A^{s_{1}}_{\alpha_{1}\alpha_{2}}\cdots A^{s_{i}l}_{\alpha_{i}\alpha_{i+1}}\cdots  A^{s_{n}}_{\alpha_{n}\alpha_{1}}\Phi^{s_{1}\cdots s_{n}}(\mathbf{x})) + \eta),
\end{align}
where $\lambda$ and $\eta$ are arbitrary real numbers.

We denote the limit of the function $\tilde{\sigma}(\cdot)$ as $S(\cdot)$. So by taking the limit of $\tilde{\sigma}(\cdot)$ as $\lambda \to \infty$, we get  
\begin{align}
S(\mathbf{x}, \eta) = 
\begin{cases}
\sigma(\eta) \quad &\mbox{if $\sum_{\{\alpha, s\}}A^{s_{1}}_{\alpha_{1}\alpha_{2}}\cdots A^{s_{n}}_{\alpha_{n}\alpha_{1}}\Phi^{s_{1}\cdots s_{n}} = 0$}, \\
C \quad &\mbox{if $\sum_{\{\alpha, s\}}A^{s_{1}}_{\alpha_{1}\alpha_{2}}\cdots A^{s_{n}}_{\alpha_{n}\alpha_{1}}\Phi^{s_{1}\cdots s_{n}} > 0$},\\
0 \quad &\mbox{if $\sum_{\{\alpha, s\}}A^{s_{1}}_{\alpha_{1}\alpha_{2}}\cdots A^{s_{n}}_{\alpha_{n}\alpha_{1}}\Phi^{s_{1}\cdots s_{n}} < 0$}.
\end{cases}
\end{align}
Since we can reparameterize $A$ and $\Phi$ in $\tilde{\sigma}(\cdot)$ and write it as $\sigma(\cdot)$, we have
\begin{align}
\lim_{\lambda\to\infty}
\int{\tilde{\sigma}(\mathbf{x}, \lambda, \eta) d\mu(\mathbf{x})} = 0
\end{align}
We know $\tilde{\sigma}(\cdot)$ converges to $S(\mathbf{x}, \eta)$ pointwisely (even uniformly) and also dominated by a bounded integrable function. So by dominated convergence theorem, we have 
\begin{align}
\label{eq: dominate_conv}
 \lim_{\lambda\to\infty} \int{\tilde{\sigma}(\mathbf{x}, \lambda, \eta)d\mu(\mathbf{x})} =& \int{\lim_{\lambda\to\infty}\tilde{\sigma}(\mathbf{x}, \lambda, \eta)}d\mu(\mathbf{x})\\
 =& \int{S(\mathbf{x}, \eta)d\mu(\mathbf{x})}\\
 =& \sigma(\eta)\mu(\mathcal{V}[\sum_{\{\alpha, s\}}A^{s_{1}}_{\alpha_{1}\alpha_{2}}\cdots A^{s_{n}}_{\alpha_{n}\alpha_{1}}\Phi^{s_{1}\cdots s_{n}} = 0]) \\
 & + C\mu(\mathcal{V}[\sum_{\{\alpha, s\}}A^{s_{1}}_{\alpha_{1}\alpha_{2}}\cdots A^{s_{n}}_{\alpha_{n}\alpha_{1}}\Phi^{s_{1}\cdots s_{n}} > 0]).
\end{align}

By letting $\eta$ go to $-\infty$ in \eqref{eq: dominate_conv}, we will get 
\begin{align}
    C\mu(\mathcal{V}[\sum_{\{\alpha, s\}}A^{s_{1}}_{\alpha_{1}\alpha_{2}}\cdots A^{s_{n}}_{\alpha_{n}\alpha_{1}}\Phi^{s_{1}\cdots s_{n}} > 0]) = 0,
\end{align}
thus both $\mu(\mathcal{V}[\sum_{\{\alpha, s\}}A^{s_{1}}_{\alpha_{1}\alpha_{2}}\cdots A^{s_{n}}_{\alpha_{n}\alpha_{1}}\Phi^{s_{1}\cdots s_{n}} = 0])$ and $\mu(\mathcal{V}[\sum_{\{\alpha, s\}}A^{s_{1}}_{\alpha_{1}\alpha_{2}}\cdots A^{s_{n}}_{\alpha_{n}\alpha_{1}}\Phi^{s_{1}\cdots s_{n}} > 0])$ vanish. It is easy to use $\prod_{i}A^{i}\Phi(\mathbf{x})$ to construct arbitrary simple functions. Since the simple function space is dense in $L^{\infty}(\mathbb{R}^{n})$, we know $\mu(\mathbf{x})$ vanishes. So we proved that the finite, signed and regular Borel measure $\mu \in M(I^{n})$ vanishes if the integral \eqref{eq: sigmoid_integral} is $0$. 
%We note here that another idea to prove above lemma is to use the fact the simple function is dense in $L^{n}$ ($L^{\infty}$). 
%By the dominated convergence theorem, we can exchange the order of the limit and integral operator, 
\end{proof}

Actually we can relax the assumption that the activation function is continuous in \ref{lemma: disc_lemma} to arbitrary measurable sigmoidal function since in the dominated convergence theorem we just need to assume that the function sequence is measurable but not continuous. 
Now we are ready to prove our main theorem \ref{the: main_the}.

\vspace{5pt}
{\bf \noindent Proof of Main Theorem \ref{the: main_the}.} 
We assume that closure of MPS function family $\bar{\mathcal{M}}$ is a proper subset of $C^{0}(I^{n})$.
In \ref{lemma: linearity_mps}, we show $\mathcal{M}$ is a vector subspace of $C^{0}(I^{n})$, so by Hahn-Banach theorem, there exists a non-trivial functional which annihilate the subspace $\bar{\mathcal{M}}$. By Riesz Representation Theorem, this means there exists a non-trivial finite signed regular Borel measure $\mu \in M(I^{n})$ which annihilate $\bar{\mathcal{M}}(I^{n})$. However, in \ref{lemma: disc_lemma}, we show that the measure $\mu$ which annihilate the functions with the $\sigma(\cdot)$ as activation \ref{eq: rescaling_sigmoid} itself is trivial. This contradicts the fact that a non-trivial measure exists which is derived from the Hahn-Banach theorem. So $\bar{\mathcal{M}}$ cannot be the proper subset of $C^{0}(I^{n})$, namely $\mathcal{M}$ is dense in $C^{0}(I^{n})$.
\hfill$\square$

\vspace{5pt}
We prove that $\mathcal{M}$ is dense in $C^{0}(I^{n})$ with respect to the supremum norm. With the same idea and the isometry between $L^{1}(I^{n})$ and $L^{\infty*}(I^{n})$, we can also show that the function space of MPS with measurable (bounded) $\sigma(\cdot)$ is dense in $L^{1}(I^{n})$ space. We can extend above conclusion to an arbitrary compact subset of $\mathbb{R}^{n}$ with the scaling transformation.

\subsection{Relation between MPS and One-Hidden-Layer Neural Networks}
\label{sec: relation_mps_neural_nets}

%We studied the representation capacity of MPS in earlier sections, an natural question gets our interest is what is the relation between MPS and neural networks? 
In this section we will compare the structure of MPS with that of Neural Networks and study the relation between them. We will show that the MPS function defined by \eqref{eq: modified_mps} is equivalent to the \textit{one-hidden-layer neural networks equipped with kernel functions}. The relation between MPS and Restricted Boltzmean Machine (RBM) and some other models was studied in \citep{deng2017quantum, cai2018approximating, chen2018equivalence}. In these work, it is shown that by carefully designed operations, MPS can be transformed into the other equivalent models such as RBM.  
Here we introduce the theorem on the equivalence between MPS and neural networks as follows, 
\begin{theorem}
\label{the: mps_equivalence}
For a MPS in the form as \eqref{eq: modified_mps} with the scale-invariant sigmoidal function $\sigma(\cdot)$ \eqref{eq: sigmoid} as activation, there exists an equivalent one-hidden-layer neural network with the kernel function $\Phi(\mathbf{x})$ which maps the input $\mathbf{x}$ into higher dimensional feature space. Here equivalence means that for arbitrary data set $\{\mathbf{x}_{i}\}$, the output of the MPS and the corresponding neural network are the same.
\end{theorem}
As is well known, symmetry arises naturally in quantum system and  it leads to a lot of important properties such as the conservation of stress tensor, the degeneracy of energy level, etc. 
%In MPS, gauge symmetry also exists since we can always make unitary transformation of each tensor $A^{s_{i}}_{\alpha_{i}\alpha_{i+1}}$
In MPS, some internal symmetries (gauge symmetries) which determines the physical properties of the many-body quantum systems exists  \citep{perez2006matrix}. Here actually we are more interested in the gauge symmetry or in other word, the 'redundancy' in the representation of MPS \citep{perez2006matrix}. The idea is that we can always plug a pair of unitary (orthogonal) matrices whose product is the identity matrix into the MPS to transform each tensor locally but keep the product invariant. To fix this freedom, we can use the Schmidt decomposition to write the MPS into the canonical form following three gauge conditions (please see \cite{perez2006matrix} for more details.
%If we do not consider the physical meaning, namely each tensor in the MPS represents one spin, 

\vspace{5pt}
{\bf \noindent Proof of Theorem \ref{the: mps_equivalence}.} 
If we ignore the physical meaning of MPS and from the data point of view, to fix the redundancy of this model, namely the bond dimension $\alpha$, we just need to contract the internal indices $\{\alpha_{i}\}$ which do not couple with the data \footnote{We note here that the operations of contraction indices $\{\alpha_{i}\}$ commute with each other, so the order of the contractions not affect the results.}. So by contracting all the internal indices in the MPS $\Psi(\mathbf{x})$ \eqref{eq: modified_mps}, we get
\begin{align}
    \Psi(\mathbf{x}) &= \sum_{l}\sigma(\sum _{\{\alpha, s\}}A^{s_{1}}_{\alpha_{1}\alpha_{2}}\cdots A^{s_{i}l}_{\alpha_{i}\alpha_{i+1}}\cdots A^{s_{n}}_{\alpha_{n}\alpha_{1}}\Phi^{s_{1}\cdots s_{n}}(\mathbf{x}))\\
    &= \sum_{l}\sigma(\sum_{\{s\}}W_{[1]}^{l s_{1}\cdots s_{n}}\Phi^{s_{1}\cdots s_{n}}(\mathbf{x})).
\end{align}
Actually there are two natural isomorphisms as follows, 
\begin{align}
    &f_{1}: W_{[1]}^{ls_{1}\cdots s_{n}} \to W_{[1]}^{ls}, \\
    &f_{2}: \Phi^{s_{1}\cdots s_{n}}(\mathbf{x}) \to \Phi^{s}(\mathbf{x}),
\end{align}
where $|s| = \prod_{i}|s_{i}|$. 

Considering the activation function in \eqref{eq: rescaling_sigmoid}, we can always rescale it to get $W_{2}\sigma(\cdot)$, then we have
\begin{align}
    \Psi(\mathbf{x}) &= \sum_{l}W_{[2]}\sigma(\sum_{s}W_{[1]}^{ls}\Phi^{s}(\mathbf{x})) \\
    &= \sum_{l}W^{l}_{[2]}\sigma(\sum_{s}W_{[1]}^{ls}\Phi^{s}(\mathbf{x})).
\end{align}

We can introduce interception $\mathbf{b}$ into the $\sigma(\cdot)$ by setting at least one component of each kernel $\phi^{s_{i}}(x^{i})$ to be $1$, and this leads to at least one component of $\Phi^{s}(\mathbf{x})$ to be $1$.
\hfill$\square$

\begin{example}
We consider a MPS as follows, 
\begin{align}
    \Psi(\mathbf{x}) = \sum_{l}\sigma(\sum_{\{\alpha, s\}}A^{s_{1}}_{(1)}A^{ls_{2}}_{(2)}A^{s_{3}}_{(3)}\phi^{s_{1}}(x_{1})\phi^{s_{2}}(x_{2})\phi^{s_{3}}(x_{3})),
\end{align}
where
\begin{align}
   &A^{s_{1}}_{(1)} = [A^{1}_{(1)}, A^{2}_{(1)}], \\
   &A^{ls_{2}}_{(2)} = 
   \begin{bmatrix}
   A^{11}_{(2)} & A^{12}_{(2)}\\
   A^{21}_{(2)} & A^{22}_{(2)}
   \end{bmatrix}, \\
   &A^{s_{3}}_{(3)} = [A^{1}_{(3)}, A^{2}_{(3)}].
   %&A^{1}_{(1)} = [a_{1}, a_{2}],
   % \quad
   %A^{2}_{(1)} = [a_{3}, a_{4}], \\
   %&A^{11}_{(2)} = 
   %\begin{bmatrix}
   %a_{5} & a_{6}\\
   %a_{7} & a_{8}
   %\end{bmatrix}, \quad
   %A^{21}_{(2)} = 
   %\begin{bmatrix}
   %a_{9} & a_{10}\\
   %a_{11} & a_{12}
   %\end{bmatrix}, \\
   %&A^{12}_{(2)} = 
   %\begin{bmatrix}
   %a_{13} & a_{14}\\
   %a_{15} & a_{16}
   %\end{bmatrix}, \quad
   %A^{22}_{(2)} = 
   %\begin{bmatrix}
   %a_{17} & a_{18}\\
   %a_{19} & a_{20}
   %\end{bmatrix}, \\
   %&A^{1}_{(3)} = [a_{21}, a_{22}], 
   %\quad
   %A^{2}_{(3)} = [a_{23}, a_{24}],  
\end{align}
We note here that $A^{i}_{(1)}$ and $A^{i}_{(3)}$ are two dimensional vectors and $A^{ij}_{(2)}$ is $2 \times 2$ matrix.
And the kernel function is as 
\begin{align}
\label{eq: x_1_kernel}
    \phi^{1}(x_{1}) = [x_{1}, 1],
    \quad
    \phi^{2}(x_{2}) = [x_{2}, 1], 
    \quad
    \phi^{3}(x_{3}) = [x_{3}, 1].
\end{align}
We can get the equivalent neural network as 
\begin{align}
    \Psi(\mathbf{x}) &= \sum_{l}\sigma(\sum_{i}W^{li}\Phi^{i})\\
    W^{li} &= 
    \begin{bmatrix}
    W^{1111} & W^{1211} & W^{1121} & W^{1112} & 
    W^{1122} & W^{1212} & 
    W^{1221} & W^{1222} \\
    W^{2111} & W^{2211} & W^{2121} & W^{2112} & 
    W^{2122} & W^{2212} & W^{2221} & W^{2222} 
    \end{bmatrix}\\
    (\Phi^{i})^{T} &= 
    \begin{bmatrix}
    \label{eq: poly_kernel}
    x_{1}x_{2}x_{3} & 
    x_{2}x_{3} & 
    x_{1}x_{3} & 
    x_{1}x_{2} &
    x_{1} &
    x_{2} &
    x_{3} &
    1
    \end{bmatrix},
\end{align}
where 
\begin{align}
    W^{ijkl} = A^{j}_{(1)}\cdot A^{ik}_{(2)}\cdot A^{l}_{(3)}, 
    \quad
    i, j, k, l \in \{ 1, 2 \} .
\end{align}
\end{example}
%Actually activated MPS works as one-hidden layer neural networks with kernel map which transform the input $\mathbf{x}$ into a higher dimensional space. 

In above example, we show that using the kernel function $\Phi^{s_{1}\cdots s_{n}}(\mathbf{x})$ \eqref{eq: x_1_kernel}, the polynomial kernel $\Phi^{s}(\mathbf{x})$ \eqref{eq: poly_kernel}. With careful design of the each component $\phi^{s_{i}}$, other interesting kernel functions can also be realized.

From above analysis, we know that the pure MPS \eqref{eq: pure_mps} is just the single-layer neural networks with the kernel function defined by $\Phi^{s}(\mathbf{x})$ which can introduce the coupling of the components of data points $\mathbf{x}^{(i)}$ into the model. We know that one-layer neural networks cannot represent arbitrary boolean function such as the 'XOR' gate, however arbitrary boolean functions can be implemented by the pure MPS thanks to the kernel function $\Phi^{s}(\mathbf{x})$. 
%which maps the input $\mathbf{x}^{(i)}$ into the higher dimensional space. 
In next example, we will show the equivalent neural network configure for the MPS which implements the $3$-input Parity Gate in \ref{exp: parity_fun}.    

\begin{example} For the MPS given in \ref{exp: parity_fun} which realizes the $3$-input Parity Gate, we can get the equivalent single-layer neural network as follows,
\begin{align}
    W^{s} &= 
    \begin{bmatrix}
    1 & 0 & 0 & 1 & 0 & 1 & 1 & 0
    \end{bmatrix}, \\
    (\Phi^{s})^{T} &=
    \begin{bmatrix}
    \Phi^{111} & \Phi^{112} & \Phi^{121} & \Phi^{122} & 
    \Phi^{211} & \Phi^{212} & \Phi^{221} & \Phi^{222} 
    \end{bmatrix},
    %\begin{bmatrix}
    %x_{1}x_{2}x_{3} & 
    %x_{1}x_{2}(1-x_{3}) & 
    %x_{1}(1-x_{2})x_{3} & 
    %x_{1}(1-x_{2})(1-x_{3}) &\\ 
    %(x_{1}-1)x_{2}x_{3} & 
    %(x_{1}-1)x_{2}(x_{3}-1) & 
    %(1-x_{1})(1-x_{2})x_{3} & 
    %(1-x_{1})(1-x_{3})(1-x_{3})
    %\end{bmatrix}
\end{align}
where 
\begin{align}
    \Phi^{ijk}(\mathbf{x}) = x_{1}^{2-i}(1-x_{1})^{i-1}x_{2}^{2-j}(1-x_{2})^{j-1}x_{3}^{2-k}(1-x_{3})^{k-1}.
\end{align}
\end{example}
\subsection{Infinitely Wide Limit of MPS from the Perspective of Neural Networks}

It is already shown that infinitely wide MPS is quivalent to the Gaussian Process (G.P.) in \citep{guo2021infinitely}. Here we can also review the relation between MPS and GP considering the equivalent relation between MPS and one-hidden-layer neural networks. 

By setting the width of the MPS \eqref{eq: pure_mps} to be infinite \footnote{At the same time, one of the bond dimension $|\alpha|$ should be set to be infinite.}, the equivalent neural network is as 
\begin{align}
    \Psi(\mathbf{x}) = \sum_{s}W^{s}\Phi^{s}(\mathbf{x}), 
\end{align}
where $|s| = \prod_{i}|s_{i}|$\footnote{In \citep{guo2021infinitely}, we also mentioned that if one $s_{i}$ goes to infinity, MPS also will converge to GP. It is easy to verify this fact here.}. So as the width $i$ goes to infinity, $|s|$ will be infinite. Since all $A^{s_{i}}_{\alpha_{i}\alpha_{i+1}}$ are independent and identically distributed (i.i.d.), it is easy to verify that all components $W^{s}$ are i.i.d. and also belong to Normal
%as $|s|$ goes to infinity, namely the infinitely wide MPS, 
by Central Limit Theorem (C.L.T.). $\Psi(\mathbf{x})$ converges to the GP defined on the data set $\{\mathbf{x}_{i}, i\in\mathbb{Z}\}$. 
%We call this as 'pure MPS' in \citep{guo2021infinitely} since no activation is introduced. 

It is already shown that one-hidden-layer neural networks converge to GP as the number of hidden neurons goes to infinity in Neal's work \citep{neal1995bayesian}. Since the activated MPS \eqref{eq: modified_mps} is equivalent to the one-hidden-layer neural networks as demonstrated in \ref{sec: relation_mps_neural_nets} and the edge $l$ is just the index of the hidden layer of the equivalent neural networks, we can show that the activated MPS will converge to GP as $l$ goes to infinity. Actually this is already proved in \citep{guo2021infinitely}. The so-called 'tensor neural networks' ('MPS hidden layer nerual Networks') in \citep{guo2021infinitely} is the same\footnote{The interception $\mathbf{b}$ does not affect the asymptotic behavior since the sum of normal random variables is still normal.} as the activated MPS \eqref{eq: modified_mps} since a tunable parameter can always be absorbed into the $\sigma(\cdot)$ \eqref{eq: rescaling_sigmoid}.

%\begin{theorem}
%$\exists \phi^{s_{1}},\cdots \phi^{s_{n}},
%A^{s_{1}}_{\alpha_{1}\alpha_{2}},\cdots A^{s_{i}}_{\alpha_{i-1}\alpha_{i}}\cdots,A^{s_{n}}_{\alpha_{n-1}\alpha_{n}},\forall f \in \Omega, \forall \bold{x}\in\{0, 1\}^{k}, \Psi(\bold{x}) = f(\bold{x}), \textit{ where }\\ \Psi(\bold{x}) = \sum_{\{s\}, \{\alpha\}}a(A^{s_{1}}_{\alpha_{1}\alpha_{2}}A^{s_{2}}_{\alpha_{2}\alpha_{3}}\cdots A^{s_{n}}_{\alpha_{n-1}\alpha_{n}}\Phi^{s_{1}\cdots s_{n}}(\bold{x}))$
%\end{theorem}
%
%Above theorems show that the hypothesis space of the tensor networks contains the boolean function space. 

%!TEX root = main.tex

\section{Conclusion}
\label{sec: conclusion}

We study the universal representation power of MPS from the perspective of the boolean function space and the continuous function space $C^{0}(I^{n})$. We show that MPS can represent arbitrary boolean functions. Moreover, MPS can approximate any continuous functions defined on the compact subspace of $R^{n}$ arbitrarily well in the sense of supremum norm. Using the duality property of $L^{p}(I^{n})$ space, we can also extend the proof to the $L^{1}(I^{n})$ space and show that the activated MPS is also dense in $L^{1}(I^{n})$ space with respect to the $L^{1}$ norm. Actually the main idea in our proof is based on the work by Cybenko, Hornik, etc. \citep{cybenko1989approximation, hornik1991approximation} and the tools we utilized in our proof mainly come from functional analysis. 

We study the relation between the MPS and the neural networks and show that the activated MPS can be represented by one-hidden-layer neural networks with specific kernel functions which are determined by the structure of the kernel $\Phi^{s_{1}, \cdots, s_{n}}$ in MPS. Since the introducing of the kernels, the coupling of the components of the data point $\mathbf{x}$ can be naturally injected into the model which is similar to the idea of 'kernel machine'. We construct the equivalent neural networks of several specific MPS models and obtain 'polynomial kernel' by setting one component of each $\phi^{s_{i}}(x_{i})$ to be $1$ as \eqref{eq: poly_kernel}.

The MPS is inspired by quantum physics and achieves great successes in different areas. Considering the quantum effects of the many-body system, MPS is an effective ansats for many-body wave functions which can approximate the exact function of the system well. Several interesting physical properties of condensed-matter systems can be studied by MPS simulation. 

The bond indices $\{\alpha_{i}, i\in\{1, \cdots, n\}\}$ which connect all the tensors $A^{s_{i}}$ control the entanglement of the tensor pairs near with each other is a key property of MPS. However if we just treat it as a model applied in statistical learning tasks, etc. where no direct physical correspondence exists, $\{\alpha_{i}\}$ will not be the advantage of this model from above analysis, instead the design of the $\Phi^{s_{1}, \cdots, s_{n}}(\mathbf{x})$ becomes subtle and important now. An extreme case is that if we just use $\phi^{s_{i}}(x_{i})=x_{i}$, then pure MPS is just the single-layer neural networks without interception term whose representation power is hugely limited. 

In MPS, the bond indices $\{\alpha_{i}\}$ does not couple with the input $\mathbf{x}$ so we can contract them without changing the model's output like we did above, however in neural networks with the nested structure as $\sum{W^{[n]}\cdot\sigma(\sum{W^{[n-1]}\cdot\sigma(\cdot)+ b^{[n-1]})} + b^{[n]}}$, every layer is activated by a non-linear function and the 'gauge symmetry' (redundancy) of the indices between two connecting layers is breaking which means we cannot contract them without knowing the input $\mathbf{x}$, namely the data $\mathbf{x}$ couples with the indices in each layer. In deep neural networks, multiple simple layers are used to transform the original data into the more complicated or higher-level feature space and the map is 'learned' from the data, however the structure of MPS implicitly maps the input into a high dimensional kernel space and also the correlation of different components of the data is automatically feed into the model.

%We guess there should be a more strict explanation from the information theory perspective.

%Moreover, by the bond indices, $O(2^{n})$ complicated Hilbert space is approximated by $O(nd^{2})$ space, where $d$ is the so-called dimension $|\alpha_{i}|$, a fixed constant. 

%However, from a application perspective in other area such as statistical learning, applied math, etc, 

%MPS might not be able to work as well as it is in many-body physics. From above analysis, we 
%Intuitively the reason is that their is no correspondence         

%!TEX root = main.tex

\section*{Acknowledgements}
The authors wish to thank David Helmbold, Hongyun Wang, J. Xavier Prochaska, Qi Gong, Torsten Ehrhardt, Hai Lin and Francois Monard for their helpful discussions. 

%Erdong Guo is grateful for the financial support by Ming-Ren Teahouse and Uncertainty Quantification LLC for this work.

\bibliographystyle{./ims}

\bibliography{reference.bib}

\begin{thebibliography}{32}
\expandafter\ifx\csname natexlab\endcsname\relax\def\natexlab#1{#1}\fi
\expandafter\ifx\csname url\endcsname\relax
  \def\url#1{\texttt{#1}}\fi
\expandafter\ifx\csname urlprefix\endcsname\relax\def\urlprefix{}\fi

\bibitem[{Biamonte(2019)}]{biamonte2019lectures}
\text{Biamonte, J.} (2019).
\newblock Lectures on quantum tensor networks.
\newblock \textit{arXiv preprint arXiv:1912.10049}.

\bibitem[{Biamonte and Bergholm(2017)}]{biamonte2017tensor}
\text{Biamonte, J.} and \text{Bergholm, V.} (2017).
\newblock Tensor networks in a nutshell.
\newblock \textit{arXiv preprint arXiv:1708.00006}.

\bibitem[{Biamonte et~al.(2011)Biamonte, Clark and
  Jaksch}]{biamonte2011categorical}
\text{Biamonte, J.~D.}, \text{Clark, S.~R.} and \text{Jaksch, D.} (2011).
\newblock Categorical tensor network states.
\newblock \textit{AIP Advances}, \textbf{1} 042172.

\bibitem[{Cai and Liu(2018)}]{cai2018approximating}
\text{Cai, Z.} and \text{Liu, J.} (2018).
\newblock Approximating quantum many-body wave functions using artificial
  neural networks.
\newblock \textit{Physical Review B}, \textbf{97} 035116.

\bibitem[{Chen et~al.(2018)Chen, Cheng, Xie, Wang and
  Xiang}]{chen2018equivalence}
\text{Chen, J.}, \text{Cheng, S.}, \text{Xie, H.}, \text{Wang, L.} and
  \text{Xiang, T.} (2018).
\newblock Equivalence of restricted boltzmann machines and tensor network
  states.
\newblock \textit{Physical Review B}, \textbf{97} 085104.

\bibitem[{Cichocki et~al.(2016)Cichocki, Lee, Oseledets, Phan, Zhao and
  Mandic}]{cichocki2016tensor}
\text{Cichocki, A.}, \text{Lee, N.}, \text{Oseledets, I.}, \text{Phan, A.-H.},
  \text{Zhao, Q.} and \text{Mandic, D.~P.} (2016).
\newblock Tensor networks for dimensionality reduction and large-scale
  optimization: Part 1 low-rank tensor decompositions.
\newblock \textit{Foundations and Trends{\textregistered} in Machine Learning},
  \textbf{9} 249--429.

\bibitem[{Cirac and Verstraete(2009)}]{cirac2009renormalization}
\text{Cirac, J.~I.} and \text{Verstraete, F.} (2009).
\newblock Renormalization and tensor product states in spin chains and
  lattices.
\newblock \textit{Journal of Physics A: Mathematical and Theoretical},
  \textbf{42} 504004.

\bibitem[{Cybenko(1989)}]{cybenko1989approximation}
\text{Cybenko, G.} (1989).
\newblock Approximation by superpositions of a sigmoidal function.
\newblock \textit{Mathematics of control, signals and systems}, \textbf{2}
  303--314.

\bibitem[{Deng et~al.(2017)Deng, Li and Sarma}]{deng2017quantum}
\text{Deng, D.-L.}, \text{Li, X.} and \text{Sarma, S.~D.} (2017).
\newblock Quantum entanglement in neural network states.
\newblock \textit{Physical Review X}, \textbf{7} 021021.

\bibitem[{Deutsch(1989)}]{deutsch1989quantum}
\text{Deutsch, D.~E.} (1989).
\newblock Quantum computational networks.
\newblock \textit{Proceedings of the Royal Society of London. A. Mathematical
  and Physical Sciences}, \textbf{425} 73--90.

\bibitem[{Efthymiou et~al.(2019)Efthymiou, Hidary and
  Leichenauer}]{efthymiou2019tensornetwork}
\text{Efthymiou, S.}, \text{Hidary, J.} and \text{Leichenauer, S.} (2019).
\newblock Tensornetwork for machine learning.
\newblock \textit{arXiv preprint arXiv:1906.06329}.

\bibitem[{Evenbly and Vidal(2011)}]{evenbly2011tensor}
\text{Evenbly, G.} and \text{Vidal, G.} (2011).
\newblock Tensor network states and geometry.
\newblock \textit{Journal of Statistical Physics}, \textbf{145} 891--918.

\bibitem[{Feynman(1986)}]{feynman1986quantum}
\text{Feynman, R.} (1986).
\newblock Quantum-mechanical computers, suc.
\newblock In \textit{Phys. Sci}, vol. 149.

\bibitem[{Folland(2013)}]{folland2013real}
\text{Folland, G.} (2013).
\newblock \textit{Real Analysis: Modern Techniques and Their Applications}.
\newblock Pure and Applied Mathematics: A Wiley Series of Texts, Monographs and
  Tracts, Wiley.
\newline\urlprefix\url{https://books.google.com/books?id=wI4fAwAAQBAJ}

\bibitem[{Gallego and Orus(2017)}]{gallego2017language}
\text{Gallego, A.~J.} and \text{Orus, R.} (2017).
\newblock Language design as information renormalization.
\newblock \textit{arXiv preprint arXiv:1708.01525}.

\bibitem[{Glasser et~al.(2018)Glasser, Pancotti and
  Cirac}]{glasser2018supervised}
\text{Glasser, I.}, \text{Pancotti, N.} and \text{Cirac, J.~I.} (2018).
\newblock Supervised learning with generalized tensor networks.
\newblock \textit{arXiv preprint arXiv:1806.05964}.

\bibitem[{Guo and Draper(2021)}]{guo2021infinitely}
\text{Guo, E.} and \text{Draper, D.} (2021).
\newblock Infinitely wide tensor networks as gaussian process.
\newblock \textit{arXiv preprint arXiv:2101.02333}.

\bibitem[{Han et~al.(2018)Han, Wang, Fan, Wang and Zhang}]{han2018unsupervised}
\text{Han, Z.-Y.}, \text{Wang, J.}, \text{Fan, H.}, \text{Wang, L.} and
  \text{Zhang, P.} (2018).
\newblock Unsupervised generative modeling using matrix product states.
\newblock \textit{Physical Review X}, \textbf{8} 031012.

\bibitem[{Hornik(1991)}]{hornik1991approximation}
\text{Hornik, K.} (1991).
\newblock Approximation capabilities of multilayer feedforward networks.
\newblock \textit{Neural networks}, \textbf{4} 251--257.

\bibitem[{LeCun et~al.(2015)LeCun, Bengio and Hinton}]{lecun2015deep}
\text{LeCun, Y.}, \text{Bengio, Y.} and \text{Hinton, G.} (2015).
\newblock Deep learning.
\newblock \textit{nature}, \textbf{521} 436--444.

\bibitem[{Neal(1995)}]{neal1995bayesian}
\text{Neal, R.~M.} (1995).
\newblock \textit{BAYESIAN LEARNING FOR NEURAL NETWORKS}.
\newblock Ph.D. thesis, University of Toronto.

\bibitem[{Novikov et~al.(2016)Novikov, Trofimov and
  Oseledets}]{novikov2016exponential}
\text{Novikov, A.}, \text{Trofimov, M.} and \text{Oseledets, I.} (2016).
\newblock Exponential machines.
\newblock \textit{arXiv preprint arXiv:1605.03795}.

\bibitem[{Or{\'u}s(2014{\natexlab{a}})}]{orus2014advances}
\text{Or{\'u}s, R.} (2014{\natexlab{a}}).
\newblock Advances on tensor network theory: symmetries, fermions,
  entanglement, and holography.
\newblock \textit{The European Physical Journal B}, \textbf{87} 1--18.

\bibitem[{Or{\'u}s(2014{\natexlab{b}})}]{orus2014practical}
\text{Or{\'u}s, R.} (2014{\natexlab{b}}).
\newblock A practical introduction to tensor networks: Matrix product states
  and projected entangled pair states.
\newblock \textit{Annals of Physics}, \textbf{349} 117--158.

\bibitem[{Penrose(1971)}]{penrose1971applications}
\text{Penrose, R.} (1971).
\newblock Applications of negative dimensional tensors.
\newblock \textit{Combinatorial mathematics and its applications}, \textbf{1}
  221--244.

\bibitem[{Perez-Garcia et~al.(2006)Perez-Garcia, Verstraete, Wolf and
  Cirac}]{perez2006matrix}
\text{Perez-Garcia, D.}, \text{Verstraete, F.}, \text{Wolf, M.~M.} and
  \text{Cirac, J.~I.} (2006).
\newblock Matrix product state representations.
\newblock \textit{arXiv preprint quant-ph/0608197}.

\bibitem[{Stinchcombe and White(1990)}]{stinchcombe1990approximating}
\text{Stinchcombe, M.} and \text{White, H.} (1990).
\newblock Approximating and learning unknown mappings using multilayer
  feedforward networks with bounded weights.
\newblock In \textit{1990 IJCNN International Joint Conference on Neural
  Networks}. IEEE.

\bibitem[{Stoudenmire and Schwab(2016)}]{stoudenmire2016supervised}
\text{Stoudenmire, E.} and \text{Schwab, D.~J.} (2016).
\newblock Supervised learning with tensor networks.
\newblock In \textit{Advances in Neural Information Processing Systems}.

\bibitem[{Verstraete et~al.(2004)Verstraete, Porras and
  Cirac}]{verstraete2004density}
\text{Verstraete, F.}, \text{Porras, D.} and \text{Cirac, J.~I.} (2004).
\newblock Density matrix renormalization group and periodic boundary
  conditions: A quantum information perspective.
\newblock \textit{Physical review letters}, \textbf{93} 227205.

\bibitem[{Vidal(2003)}]{vidal2003efficient}
\text{Vidal, G.} (2003).
\newblock Efficient classical simulation of slightly entangled quantum
  computations.
\newblock \textit{Physical review letters}, \textbf{91} 147902.

\bibitem[{Vidal(2007)}]{vidal2007classical}
\text{Vidal, G.} (2007).
\newblock Classical simulation of infinite-size quantum lattice systems in one
  spatial dimension.
\newblock \textit{Physical review letters}, \textbf{98} 070201.

\bibitem[{White(1992)}]{white1992density}
\text{White, S.~R.} (1992).
\newblock Density matrix formulation for quantum renormalization groups.
\newblock \textit{Physical review letters}, \textbf{69} 2863.

\end{thebibliography}

\clearpage
%!TEX root = main.tex

\appendix{}
%\section{Representation of Boolean Functions}\label{sec: rep_boolean_fun}
%
%\subsection{Proof of Theorem \ref{the: and_represent}}\label{prf: and_represent}
%
%\begin{proof}
%
%Here we prove the above theorem directly. Namely, we just need to construct the MPS $A^{s_{1}}_{\alpha_{1}}$, $B^{s_{2}}_{\alpha_{2}}$ and the kernel functions $\phi^{s_{1}}$ and $\phi^{s_{2}}$ to show the equivalence holds. 
%
%We set $|s_{1}| = 2$, $|\alpha_{1}| = 2$, 
%\begin{align*}
%    &\phi^{s_{1}}(x_{i}) = [x_{i}, 1-x_{i}],\\
%    &A^{s_{1}}_{\alpha_{1}} = [[w, w], [-2w, w]],\\
%    &B^{s_{2}}_{\alpha_{2}} = [[w, w], [w, -2w]].\\
%    & w\in \mathbb{R}
%\end{align*}
%\end{proof}

\end{document}